\documentclass[a4paper]{article}

\usepackage{geometry}
\usepackage{times}
\usepackage{soul}
\usepackage[utf8]{inputenc}
\usepackage[small]{caption}
\usepackage{amsmath}
\newtheorem{theorem}{Theorem}

\newtheorem{assumption}[theorem]{Assumption}

\newtheorem{proof}{Proof}[section]
\usepackage{amssymb}
\usepackage{subfig}
\usepackage{graphicx}
\usepackage{times}
\usepackage{bm}
\usepackage[vlined,boxed,ruled]{algorithm2e}
\usepackage{sansmath}
\usepackage{hyperref}

\DeclareSymbolFont{extraup}{U}{zavm}{m}{n}
\DeclareMathSymbol{\varheart}{\mathalpha}{extraup}{86}

\usepackage{listings} 

\newcommand{\PP}{\mathbb{P}}
\newcommand{\EE}{\mathbb{E}}

\title{Policy Optimization with Second-Order Advantage Information}

\author{
        Jiajin Li\thanks{These authors contribute equally to this work.}\\
        \href{mailto:jjli@se.cuhk.edu.hk}{jjli@se.cuhk.edu.hk}\\
        The Chinese University of Hong Kong 
        \and
        Baoxiang Wang\footnotemark[1]  \\
       \href{mailto:bxwang@cse.cuhk.edu.hk}{bxwang@cse.cuhk.edu.hk}\\
        The Chinese University of Hong Kong
}
\date{\vspace{-4ex}}

\begin{document}
\maketitle

\begin{abstract}
%background
Policy optimization on high-dimensional continuous control tasks exhibits its difficulty caused by the large variance of the policy gradient estimators.
%ASDG
We present the action subspace dependent gradient (ASDG) estimator which incorporates the Rao-Blackwell theorem (RB) and Control Variates (CV) into a unified framework to reduce the variance.
%POSA
To invoke RB, our proposed algorithm (POSA) learns the underlying factorization structure among the action space based on the second-order advantage information. 
%deepfm
POSA captures the quadratic information explicitly and efficiently by utilizing the wide \& deep architecture.
%experiment
Empirical studies show that our proposed approach demonstrates the performance improvements on high-dimensional synthetic settings and OpenAI Gym's MuJoCo continuous control tasks.
\end{abstract}

\section{Introduction}

Deep reinforcement learning (RL) algorithms have been widely applied in various challenging problems, including video games \cite{mnih2015human}, board games \cite{silver2017mastering}, robotics \cite{levine2016end}, dynamic routing \cite{wu2017framework,li2016contextual}, and continuous control tasks \cite{schulman2017proximal,lillicrap2015continuous}.
An important approach among these methods is policy gradient (PG).
Since its inception \cite{williams1992simple}, PG has been continuously improved by the Control Variates (CV) \cite{oates2017control} theory.
Examples are REINFORCE \cite{williams1992simple}, Advantage actor-critic (A2C) \cite{mnih2016asynchronous}, Q-prop \cite{gu2016q}, and action-dependent baselines \cite{liu2018action-dependent,grathwohl2018backpropagation,tucker2018mirage}.
However, when dealing with high-dimensional action spaces, CV has limited effects regarding the sample efficiency.
Rao-Blackwell theorem (RB) \cite{casella1996rao}, though not heavily adopted in policy gradient, is commonly used with CV to address high-dimensional spaces  \cite{ranganath2014black}.
%In fact, RB reduces the variance several orders of magnitude more than CV do, as demonstrated in the block box variational inference model \cite{ranganath2014black}.

Motivated by the success of RB in high-dimensional spaces \cite{ranganath2014black}, we incorporate both RB and CV into a unified framework.
We present the action subspace dependent gradient (ASDG) estimator. 
ASDG first breaks the original high dimensional action space into several low dimensional action subspaces and replace the expectation (i.e., policy gradient) with its conditional expectation over subspaces (RB step) to reduce the sample space.
A baseline function associated with each of the corresponding action subspaces is used to further reduce the variance (CV step).
While ASDG is benefited from both RB and CV's ability to reduce the variance, we show that ASDG is unbiased under relatively weak assumptions over the advantage function.

The major difficulty to invoke RB is to find a satisfying action domain partition.
Novel trials such as \cite{wu2018variance} utilize RB under the conditional independence assumption which assumes that the policy distribution is fully factorized with respect to the action.
Whilst it dramatically reduces the estimation variance, such a strong assumption limits the policy distribution flexibility and \cite{wu2018variance} is conducting the optimization in a restricted domain.
In our works, we show that Hessian of the advantage with respect to the action is theoretically connected with the action space structure.
Specifically, the block-diagonal structure of Hessian is corresponding to the partition of the action space.
We exploit such second-order information with the evolutionary clustering algorithm \cite{chakrabarti2006evolutionary} to learn the underlying factorization structure in the action space.
Instead of the vanilla multilayer perceptron, we utilize the wide \& deep architecture \cite{cheng2016wide} to capture such information explicitly and efficiently. 
With the second-order advantage information, ASDG finds the partition that approximates the underlying structure of the action space.

We evaluate our method on a variety of reinforcement learning tasks, including a high-dimensional synthetic environment and several OpenAI Gym’s MuJoCo continuous control environments. 
We build ASDG and POSA on top of proximal policy optimization (PPO), and demonstrate that ASDG consistently obtains the ideal balance: while improving the sample efficiency introduced by RB \cite{wu2018variance}, it keeps the accuracy of the feasible solution \cite{liu2018action-dependent}. 
In environments where the model assumptions are satisfied or minimally violated empirically, while not trivially satisfied by \cite{wu2018variance}, POSA outperforms previous studies with the overall cumulated rewards it achieves.
In the continuous control tasks, POSA is either competitive or superior, depending on whether the action space exhibits its structure under the environment settings. 

\section{Background}

\subsection{Notation}
We present the canonical reinforcement learning (RL) formalism in this section.
Consider policy learning in the discrete-time Markov decision process (MDP) defined by the tuple $(\mathcal{S},\mathcal{A}, \mathcal{T}, r,\rho_0,\gamma)$ where $\mathcal{S} \in \mathbb{R}^n$ is the $n$ dimensional state space, $\mathcal{A} \in \mathbb{R}^m$ is the $m$ dimensional action space, $\mathcal{T}:\mathcal{S}\times\mathcal{A}\times\mathcal{S} \to \mathbb{R}^+$ is the environment transition probability function, $r:\mathcal{S}\times\mathcal{A}\to\mathbb{R}$ is the reward function, $\rho_0$ is the initial state distribution and $\gamma \in (0,1]$ is the unnormalized discount factor. RL learns a stochastic policy $\pi_\theta:\mathcal{S}\times\mathcal{A}\to\mathbb{R}_+$, which is parameterized by $\theta$, to maximize the expected cumulative reward
\[J(\theta) = \EE_{s\sim\rho_\pi, a \sim \pi}[\sum_{t=0}^\infty \gamma^tr(s_t,a_t)].\]
In the above equation, $\rho_\pi(s) =\sum_{t=1}^\infty \gamma^{t-1} \PP(s_t =s)$ is the discounted state visitation distribution. 
Define the value function
\[V^\pi(s_t) = \EE_{\pi}[\sum_{t'\geq t}^\infty \gamma^{t'-t}r(s_{t'},a_{t'})|s_t,\pi]\]
to be the expected return of policy $\pi$ at state $s_t$. Define the state-action function
\[Q^\pi(s_t,a_t) = \EE_{\pi}[\sum_{t'\geq t}^\infty \gamma^{t'-t}r(s_{t'},a_{t'})|s_t,a_t,\pi]\]
to be the expected return by policy $\pi$ after taking the action $a_t$ at the state $s_t$. We use $\hat{Q}^\pi(s_t,a_t)$ and $\hat{V}^\pi(s_t)$ to denote the empirical function approximator of $Q^\pi(s_t,a_t)$ and $V^\pi(s_t)$, respectively. Define the advantage function to be the gap between the value function and the action-value, as $A^\pi(s_t,a_t) = Q^\pi(s_t,a_t)-V^\pi(s_t)$.
To simplify the notation, we focus on the time-independent formulation $J(\theta)=\EE_{\pi,\rho_\pi}[r(s,a)]$. 
According to the policy gradient theorem \cite{williams1992simple}, the gradient of the expected cumulative reward can be estimated as
\[\nabla_\theta J(\theta) = \EE_\pi[\nabla_\theta\log \pi(a|s)Q^\pi(s,a)].\]

\subsection{Variance Reduction Methods}

In practice, the vanilla policy gradient estimator is commonly estimated using Monte Carlo samples. A significant obstacle to the estimator is the sample efficiency. We review three prevailing variance reduction techniques in Monte Carlo estimation methods, including Control Variates, Rao-Blackwellization, and Reparameterization Trick.

\textbf{Control Variates} - Consider the case we estimate the expectation $\EE_{p(x)}[h(x)]$ with Monte Carlo samples $\{x_i\}_{i=1}^B$ from the underlying distribution $p(x)$. Usually, the original Monte Carlo estimator has high variance, and the main idea of Control Variates is to find the proper baseline function $g(x)$ to partially cancel out the variance. A baseline function $g(x)$ with its known expectation over the distribution $p(x)$ is used to construct a new estimator
\[ \hat{h}(x) = h(x)-\eta(g(x)-\EE_{p}[g(x)]),\]
where $\eta$ is a constant determined by the empirical Monte Carlo samples. The Control Variates method is unbiased but with a smaller variance $Var(\hat{h}(x)) \leq Var(h(x))$ at the optimal value $\eta^\ast=\frac{Cov(h,g)}{Var(g)}$.

\textbf{Rao-Blackwellization} - Though most of the recent policy gradient studies reduce the variance by Control Variates, the Rao-Blackwell theorem \cite{casella1996rao} decreases the variance significantly more than CV do, especially in high-dimensional spaces \cite{ranganath2014black}. The motivation behind RB is to replace the expectation with its conditional expectation over a subset of random variables. In this way, RB transforms the original high-dimensional integration computation problem into estimating the conditional expectation on several low-dimensional subspaces separately. 

Consider a simple setting with two random variable sets $\mathcal{A}$ and $\mathcal{B}$ and the objective is to compute the expectation $\EE[h(\mathcal{A},\mathcal{B})]$. Denote that the conditional expectation $\hat{\mathcal{B}}$ as $\hat{\mathcal{B}}=\EE[h(\mathcal{A},\mathcal{B})|\mathcal{A}]$. The variance inequality
\[Var(\hat{\mathcal{B}}) \leq Var(h(\mathcal{A},\mathcal{B}))\]
holds as shown in the Rao-blackwell theorem. In practical, when $\mathcal{A}$ and $\mathcal{B}$ are in high dimensional spaces, the conditioning is very useful and it reduces the variance significantly. The case of multiple random variables is hosted in a similar way.

\textbf{Reparameterization Trick} - One of the recent advances in variance reduction is the reparameterization trick. It provides an estimator with lower empirical variance compared with the score function based estimators, as demonstrated in \cite{kingma2013auto,ranganath2016hierarchical}. Using the same notation as is in the Control Variates section, we assume that the random variable $x$ is reparameterized by $x = f(\theta,\xi), \xi \sim q(\xi)$, where $q(\xi)$ is the base distribution (e.g., the standard normal distribution or the uniform distribution). 
Under this assumption, the gradient of the expectation $\EE_{p(x)}[h(x)]$ can be written as two identical forms i.e., the score function based form and reparameterization trick based form
\begin{equation}
\label{eq1}
\EE_p[\nabla_\theta \log p(x)h(x)] = \EE_q[\nabla_\theta f(\theta,\xi) \nabla_x h(x)].
\end{equation}
The reparameterization trick based estimator (the right-hand side term) has relatively lower variance. 
Intuitively, the reparameterization trick provides more informative gradients by exposing the dependency of the random variable $x$ on the parameter $\theta$. 
%In contrast, the score function gradient estimator only depends on the log-density function.
\subsection{Policy Gradient Methods}
Previous attempts to reduce the variance mainly focus on the Control Variates method in the policy gradient framework (i.e., REINFORCE, A2C, Q-prop). 
A proper choice of the baseline function is vital to reduce the variance. 
The vanilla policy gradient estimator, REINFORCE \cite{williams1992simple}, subtracts the constant baseline from the action-value function, 
\[\nabla_\theta J(\theta)_{RF} = \EE_\pi[\nabla_\theta\log \pi(a|s)(Q^\pi(s,a)-b)].\]

The estimator in REINFORCE is unbiased. The key point to conclude the unbiasedness is that the constant baseline function has a zero expectation with the score function. 
Motivated by this, the baseline function is set to be the value function $V^\pi(s)$ in the advantage actor-critic (A2C) method \cite{mnih2016asynchronous}, as the value function can also be regarded as a constant under the policy distribution $\pi(a|s)$ with respect to the action $a$. 
Thus the A2C gradient estimator is
\begin{align}
\nabla_\theta J(\theta)_{A2C} &=\EE_\pi[\nabla_\theta\log\pi(a|s)(Q^\pi(s,a)-V^\pi(s))]\notag\\
&=\EE_\pi[\nabla_\theta\log\pi(a|s)A^\pi(s,a)].\notag
\end{align}

To further reduce the gradient estimate variance to acquire a zero-asymptotic variance estimator, \cite{liu2018action-dependent} and \cite{grathwohl2018backpropagation} propose a general action dependent baseline function $b(s,a)$ based on the identity \eqref{eq1}. 
Note that the stochastic policy distribution $\pi_\theta(a|s)$ is reparametrized as $a = f(\theta,s,\xi),\xi\sim q(\xi)$, we rewrite Eq.~\eqref{eq1} to get a zero-expectation baseline function as below
\begin{equation}
\label{eq2}
\EE[\nabla_\theta\log\pi(a|s)b(s,a)-\nabla_\theta f(\theta,s,\xi)\nabla_a b(s,a)] = 0.
\end{equation}
Incorporating with the zero-expectation baseline \eqref{eq2}, the general action dependent baseline (GADB) estimator is formulated as
\begin{align}
\label{eq3}
\nabla_\theta J(\theta)_{GADB} = &\EE_\pi[\nabla_\theta\log\pi(a|s)(Q^\pi(s,a)-b(s,a))+\nabla_\theta f(\theta,s,\xi)\nabla_a b(s,a)].
\end{align}
\section{Methods}
\subsection{Construct the ASDG Estimator}
We present our action subspace dependent gradient (ASDG) estimator by applying RB on top of the GADB estimator.
Starting with Eq.~\eqref{eq3}, we rewrite the baseline function in the form of $b(s,a) = V^\pi(s) + c(s,a)$. The GADB estimator in Eq.~\eqref{eq3} is then formulated as
\begin{align*}
\label{eq4}
\nabla_\theta J(\theta)_{GADB} & = \EE_\pi[\nabla_\theta\log\pi(a|s)(A^\pi(s,a)-c(s,a))+\nabla_\theta f(\theta,s,\xi) \nabla_a c(s,a)].
\end{align*}

\begin{assumption}[Advantage Quadratic Approximation]
\label{as1}
Assume that the advantage function $A^\pi(s,a)$ can be locally second-order Taylor expanded with respect to $a$ at some point $a^\ast$, that is,
\begin{align}
A^\pi(a,s) \approx & A^\pi(a^\ast,s)+\nabla_a A^\pi(a,s)|_{a=a^\ast}^T(a-a^\ast) \notag \\ &+\frac{1}{2}(a-a^\ast)^T\nabla_{aa}A^\pi(a,s)|_{a=a^\ast}(a-a^\ast).
\end{align}
The baseline function $c(s,a)$ is chosen from the same family.
\end{assumption}

\begin{assumption}[Block Diagonal Assumption]
\label{as2}
Assume that the row-switching transform of Hessian $\nabla_{aa}A^\pi(a,s)|_{a=a^*}$ is a block diagonal matrix $diag(M_1,\dots,M_k)$, where $\sum_{k=1}^K\text{dim}(M_k)=m$.
\end{assumption}

Based on Assumption \eqref{as1} and \eqref{as2}, the advantage function $A^\pi(s,a)$ can be divided into $K$ independent components
\[ A^\pi(s,a) = \sum\limits_{k=1}^K A^\pi_k(s,a_{(k)}),\]
where $a_{(k)}$ denotes the projection of the action $a$ to the $k$-th action subspace corresponding to $M_k$. The baseline function $c(s,a)$ is divided in the same way.

\begin{theorem}[ASDG Estimator]
\label{theorem1}
If the advantage function $A^\pi(s,a)$ and the baseline function $c(s,a)$ satisfy Assumption \eqref{as1} and \eqref{as2}, the ASDG estimator $\nabla_\theta J(\theta)_{ASDG} $ is
\begin{align}
\label{gas}
\nabla_\theta J(\theta)_{ASDG} =&  \sum\limits_{k=1}^K \EE_{\pi(a_{(k)}|s)}[\nabla_\theta\log\pi(a_{(k)}|s)(A^\pi(s,a_{(k)})
-c(s,(a_{(k)},\tilde{a}_{(-k)})))
-\nabla_\theta f_k(\theta,s,\xi)\nabla_{a_{(k)}} c_k(s,a_{(k)})],\notag
\end{align}
where $\nabla_\theta f(\theta,s,\xi) \in \mathbb{R}^{N_\theta \times m}$ is divided into $K$ parts as $\nabla_\theta f = [\nabla_\theta f_1,...,\nabla_\theta f_K]$ and $N_\theta$ is the dimension of $\theta$.
\end{theorem}

% short version 
% \begin{proof}
% Using the fact:
% \[E_{\pi(a|s)}[.] = E_{\pi(a_{(k)}|s)}E_{\pi(a_{(-k)}|a_{(k)},s)}[.]\]
% where $a_{(-k)}$ represents the dimension within action $a$ that are complementary to $a_{(k)}$. With the assumptions we have, we can obtain the ASDG estimator.
% \end{proof}
\begin{proof}
Using the fact that
\[\EE_{\pi(a|s)}[.] = \EE_{\pi(a_{(k)}|s)}\EE_{\pi(a_{(-k)}|a_{(k)},s)}[.],\]
where $a_{(-k)}$ represents the elements within $a$ that are complementary to $a_{(k)}$. With the assumptions we have
\begin{align}
\nabla J(\theta)_{ASDG} = & \EE_{\pi(a_{(k)}|s)}\EE_{\pi(a_{(-k)}|a_{(k)},s)}[(\nabla_\theta \log\pi(a_{(k)}|s)+\nabla_{\theta}\log\pi(a_{(-k)}|a_{(k)},s)) \notag\\
&(A^\pi_k(s,a_{(k)}) + \sum_{i\neq k}A^\pi_i(s,a_{(i)})- c_k(s,a_{(k)})-\sum_{i\neq k}c_i(s,a_{(i)}))\notag\\
&+\sum_{k=1}^K \nabla_\theta f_k(s,a_{(k)})\nabla_{a_{(k)}}c_k(s,a_{(k)})]\notag\\
%part(1)
= & \EE_{\pi(a_{(k)}|s)} \EE_{\pi(a_{(-k)}|a_{(k)},s)}[\nabla_\theta\log\pi(a_{(k)}|s)(A^\pi_k-c_k)-\nabla_\theta f_k\nabla_{a_{(k)}} c_k]\notag\\
& + \EE_{\pi(a_{(k)}|s)} \EE_{\pi(a_{(-k)}|a_{(k)},s)}[\nabla_\theta\log\pi(a_{(k)}|s)(\sum_{i\neq k}A^\pi_i-\sum_{i\neq k}c_i)]\label{eq6}\\
& + \EE_{\pi(a_{(k)}|s)} \EE_{\pi(a_{(-k)}|a_{(k)},s)}[\nabla_\theta\log\pi(a_{(-k)}|a_{(k)},s)(A^\pi_k-c_k)]\label{eq7}\\
& + \EE_{\pi(a_{(k)}|s)} \EE_{\pi(a_{(-k)}|a_{(k)},s)}[\nabla_\theta\log\pi(a_{(-k)}|a_{(k)},s)((\sum_{i\neq k}A^\pi_i-\sum_{i\neq k}c_i))-\sum_{i\neq k}\nabla_\theta f_i\nabla_{a_{(i)}} c_i]\notag \\
\stackrel{(\clubsuit)}{=}&\EE_{\pi(a_{(k)}|s)}[\nabla_\theta\log\pi(a_{(k)}|s)(A^\pi_k-c_k)-\nabla_\theta f_k\nabla_{a_{(k)}} c_k] \notag\\
& + \EE_{\pi(a_{(-k)}|a_{(k)},s)}[\nabla_\theta\log\pi(a_{(-k)}|a_{(k)},s)((\sum_{i\neq k}A^\pi_i-\sum_{i\neq k}c_i))-\sum_{i\neq k}\nabla_\theta f_i\nabla_{a_{(i)}} c_i]\label{eq8}\notag\\
\stackrel{(\varheart)}{=}&\sum\limits_{k=1}^K \EE_{\pi(a_{(k)}|s)}[\nabla_\theta\log\pi(a_{(k)}|s)(A^\pi_k-c_k)-\nabla_\theta f_k\nabla_{a_{(k)}} c_k] \notag \\
=& \sum\limits_{k=1}^K \EE_{\pi(a_{(k)}|s)}[\nabla_\theta\log\pi(a_{(k)}|s)(A^\pi_k + \sum_{i\neq k}A^\pi_i-c_k-\sum_{i\neq k}c_i)-\nabla_\theta f_k\nabla_{a_{(k)}} c_k] \notag \\
=& \sum\limits_{k=1}^K \EE_{\pi(a_{(k)}|s)}[\nabla_\theta\log\pi(a_{(k)}|s)(A^\pi(s,a)-c(s,a_{(k)},\tilde{a}_{(-k)}))-\nabla_\theta f_k\nabla_{a_{(k)}} c_k],
\end{align}

where $(\clubsuit)$ holds as term \eqref{eq6} and term \eqref{eq7} equal to zero (using the property that the expectation of the score function is zero) and $(\varheart)$ is expanded by induction. $\hfill\blacksquare$

\end{proof}
Our assumptions are relatively weak compared with previous studies on variance reduction for policy optimization. Different from the fully factorization policy distribution assumed in \cite{wu2018variance}, our method relaxes the assumption to the constraints on the advantage function $A^\pi(s,a)$ with respect to the action space instead. Similar to that, we just use this assumption to obtain the structured factorization action subspaces to invoke the Rao-Blackwellization and our estimator does not introduce additional bias.

\textbf{Connection with other works} - If we assume the Hessian matrix of the advantage function has no block diagonal structure under any row switching transformation (i.e., $K=1$), ASDG in Theorem. \ref{theorem1} is the one inducted in \cite{liu2018action-dependent} and \cite{grathwohl2018backpropagation}. If we otherwise assume that Hessian is diagonal (i.e., $K=m$), the baseline function $c(s,a_{(k)},\tilde{a}_{(-k)})$ equals to $\sum_{i\neq k} c_i(s,a_{(i)})$, which means that each action dimension is independent with its baseline function. Thus, the estimator in \cite{wu2018variance} is obtained. 

\textbf{Selection of the baseline functions $c(s,a)$} - Two approaches exist to find the baseline function, including minimizing the variance of the PG estimator or minimizing the square error between the advantage function and the baseline function \cite{liu2018action-dependent,grathwohl2018backpropagation}.
Minimizing the variance is hard to implement in general, as it involves the gradient of the score function with respect to the baseline function parameter.
In our work, we use a neural network advantage approximation as our baseline function by minimizing the square error. 
Under the assumption that the variance of reparametrization term $\nabla_\theta f_k(\theta,s,\xi)\nabla_{a_{(k)}} c_k(s,a_{(k)})$ is closed to zero, the two methods yield the same result.

\subsection{Action Domain Partition with Second-Order Advantage Information}

When implementing the ASDG estimator, Temporal Difference (TD) learning methods such as Generalized Advantage Estimation (GAE) \cite{degris2012off,schulman2015high} allow us to obtain the estimation $\hat{A}(s,a)$ based on the value function $V^w(s)$ via
\begin{equation}
\label{gae}
\hat{A}(s_t,a_t) = \sum_{t^\prime \geq t}^T (\lambda\gamma)^{t^\prime-t} \delta_{t^\prime},
\end{equation}
where
\begin{equation}
\delta_t = \EE[r_t+\gamma V^w(s_{t+1})-V^w(s_t)]
\end{equation}
and $\lambda$ is the discount factor of the $\lambda$-return in GAE.
GAE further reduces the variance and avoids the action gap at the cost of a small bias.

Obviously, we cannot obtain the second-order information $\nabla_{aa}A(s,a)$ with the advantage estimation in GAE identity \eqref{gae}. 
Hence, apart from the value network $V^w(s)$, we train a separate advantage network to learn the advantage information. 
The neural network approximation $A^\mu(s,a)$ is used to smoothly interpolate the realization values $\hat{A}(s,a)$, by minimizing the square error 
\begin{equation}
\label{opt1}
\min\limits_{\mu} ||\hat{A}(s,a)-A^\mu(s,a)||^2.
\end{equation}
As shown in assumption \eqref{as2}, we use the block diagonal matrix to approximate the Hessian matrix and subsequently obtain the structure information in the action space. 
In the above advantage approximation setting, the Hessian computation is done by first approximating the advantage realization value and then differentiating the advantage approximation to obtain an approximate Hessian. 
However, for any finite number of data points there exists an infinite number of functions, with arbitrarily satisfied Hessian and gradients, which can perfectly approximate the advantage realization values \cite{li2017gradient}. 
Optimizing such a square error objective leads to unstable training and is prone to yield poor results. 
To alleviate this issue, we propose a novel wide \& deep architecture \cite{cheng2016wide} based advantage net. 
In this way, we divide the advantage approximator into two parts, including the quadratic term and the deep component, as
\begin{equation*}
A^\mu(s,a) = \beta_1 \cdot A_{wide} + \beta_2 \cdot A_{deep},
\end{equation*}
where $\beta_1$ and $\beta_2$ are the importance weights.
Subsequently, we make use of Factorization Machine (FM) model as our wide component
\begin{equation*}
A_{wide}(s,a) = w_0(s) + w_1(s)^Ta + w_2(s) w_2(s)^T \odot aa^T,
\end{equation*}
where $w_0(s) \in \mathbb{R}$, $w_1(s) \in \mathbb{R}^{m}$ and $w_2(s)\in \mathbb{R}^{m\times m^\prime}$ are the coefficients associated with the action. Also, $m^\prime$ is the dimension of latent feature space in the FM model. Note that the Hadamard product $A \odot B = \sum_{i,j} A_{ij}B_{ij}$. 
To increase the signal-to-noise ratio of the second-order information, we make use of wide components Hessian $w_2(s)w_2(s)^T$ as our Hessian approximator in POSA. 
The benefits are two-fold.
On the one hand, we can compute the Hessian via the forward propagation with low computational costs. 
On the other hand, the deep component involves large noise and uncertainties and we obtain stable and robust Hessian by excluding the deep component from calculating Hessian. 

The Hessian matrix contains both positive and negative values. 
However, we concern only the pairwise dependency between the action dimensions, which can be directly represented by the absolute value of Hessian. 
For instance, considering a quadratic function $f(x) = a + b^Tx + x^TCx, x\in \mathbb{R}^m$, it can be written as $f(x) = a + \sum_{i} b_ix_i + \sum_{i,j} C_{ij} x_ix_j$. 
The elements in the Hessian matrix satisfy $\frac{\partial^2f(x)}{\partial x_i\partial x_j} = C_{ij}$. When $C_{ij}$ is close to zero, $x_i$ and $x_j$ are close to be independent. 
Thus we can decompose the function $f(x)$ accordingly optimize the components separately.

We modify the evolutionary clustering algorithm in \cite{chakrabarti2006evolutionary} by using the absolute approximating Hessian $|w_2(s) w_2(s)^T|$ as the affinity matrix in the clustering task. 
In other words, each row in the absolute Hessian is regarded as a feature vector of that action dimension when running the clustering algorithm. With the evolutionary clustering algorithm, our policy optimization with second-order advantage information algorithm (POSA) is described in Alg.\eqref{algo}. 
%In theorem \eqref{theorem1}, it involves two advantage information approximation terms to construct the ASDG estimator. 

\begin{algorithm}[!htbp]
\SetAlgoLined
\caption{Policy Optimization with Second-Order Advantage Information (POSA)}
\KwIn{number of iterations $N$, number of value iterations $M_w$, batch size $B$, number of subspaces $K$, initial policy parameter $\theta$, initial value and advantage parameters $w$ and $\mu$\;}
\KwOut{Policy optimal parameter $\theta$ }
\For{each iteration $n$ in $[N]$}{
Collect a batch of trajectory data  $\{s_t^{(i)}, a_t^{(i)},r_t^{(i)}\}_{i=1}^B$ \;
\For{$M_\theta$ iterations}{
	Update $\theta$ by one SGD step using PPO with ASDG in Theorem \eqref{theorem1}\;
    }
\For{$M_w$ iterations}{
	Update $w$ and $\mu$ by minimizing $||V^w(s_t)-R_t||_2^2$ and $||\hat{A}(s_t,a_t)-A^\mu(s_t,a_t)||_2^2$ in one SGD step \;
    }
Estimate $\hat{A}(s_t,a_t)$ using $V^w(s_t)$ by GAE \eqref{gae}\;
Calculate the action subspace partition $a_{(k)}$ based on the absolute Hessian $|w_2(s) w_2(s)^T|$ by the evolutionary clustering algorithm\;
}
\label{algo}
\end{algorithm}

\section{Experiments and Results}

We demonstrate the sample efficiency and the accuracy of ASDG and Alg.\eqref{algo} in terms of both performance and variance. ASDG is compared with several of the state-of-the-art gradient estimators. 
\begin{itemize}
%\item \textbf{Advantage actor critic (A2C)} \cite{mnih2016asynchronous} uses the value functions as the baselines for variance reduction, which is the only action-independent baseline, namely only stated based. 
\item \textbf{Action dependent factorized baselines (ADFB)} \cite{wu2018variance} assumes fully factorized policy distributions, and uses $A(s,(\bar{a}_{(k)}, a_{(-k)}))$ as the $k$-th dimensional baseline. The subspace $a_{(k)}$ is restricted to contain only one dimension, which is the special case of ASDG with $K=m$.
\item \textbf {Generalized advantage dependent baselines (GADB)} \cite{liu2018action-dependent,grathwohl2018backpropagation} uses a general baseline function $c(s,a)$ which depends on the action. It does not utilize Rao-Blackwellization and is our special case when $K=1$.
\end{itemize}

\subsection{Implementation Details}

Our algorithm is built on top of PPO where the advantage realization value is estimated by GAE. Our code is available at \href{https://github.com/wangbx66/Action-Subspace-Dependent}{https://github.com/wangbx66/Action-Subspace-Dependent}. We use a policy network for PPO and a value network for GAE that have the same architecture as is in \cite{mnih2016asynchronous,schulman2017proximal}. We utilize a third network which estimates the advantage $A^\mu(s,a)$ smoothly by solving Eq.~\eqref{opt1} to be our baseline function $c(s,a)$. The network computes the advantage and the Hessian matrix approximator $w_2(s)w_2(s)^T$ by a forward propagation. It uses the wide \& deep architecture. For the wide component, the state is mapped to $w_1(s)$ and $w_2(s)$ through two-layer MLPs, both with size 128 and $\tanh(\cdot)$ activation. The deep component $A_{deep}$ is a three-layer MLPs with size 128 and $\tanh(\cdot)$ activation. Our other parameters are consistent with those in \cite{schulman2017proximal} except that we reduce the learning rate by ten times (i.e., $3\cdot 10^{-4}$) for more stable comparisons.

\subsection{Synthetic High-Dimensional Action Spaces}

We design a synthetic environment with a wide range of action space dimensions and explicit action subspace structures to test the performance of Alg.\eqref{algo} and compare that with previous studies. The environment is a one-step MDP where the reward $r(s,a)=\sum_{k=1}^K a_{(k)}^TM_ka_{(k)} + \epsilon$ does not depend on the state $s$ (e.g., $\epsilon$ is a random noise). In the environment, the action is partitioned into $K$ independent subspaces with a stationary Hessian of the advantage function. Each of the subspace can be regarded as an individual agent. The environment setting satisfies both Assumption \eqref{as1} and \eqref{as2}.

\begin{figure*}[h]
\centering 
	\subfloat[Dim=4, K=2]{
	\includegraphics[width=0.44\textwidth]{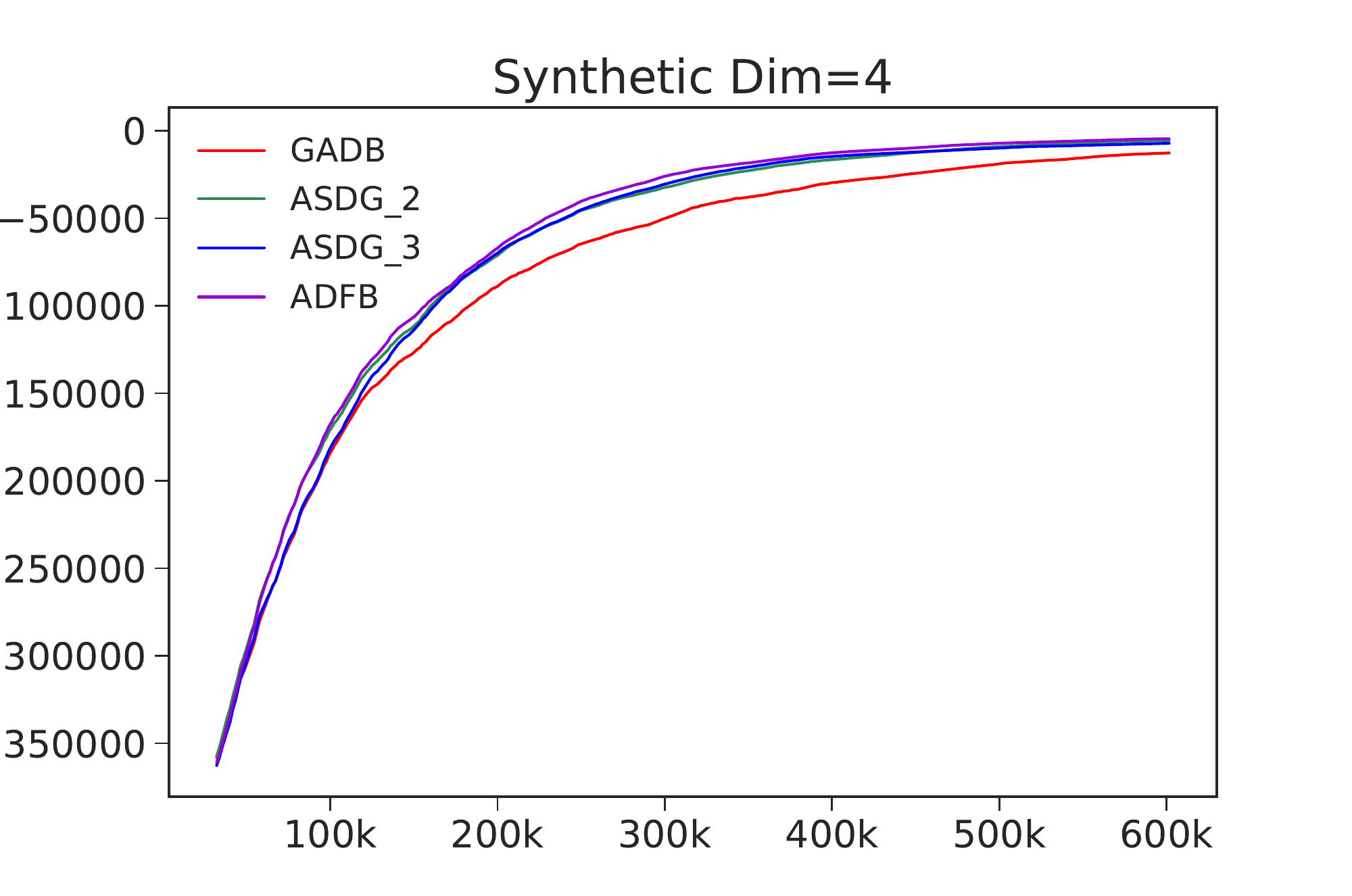}}
	\subfloat[Dim=10, K=2]{
	\includegraphics[width=0.44\textwidth]{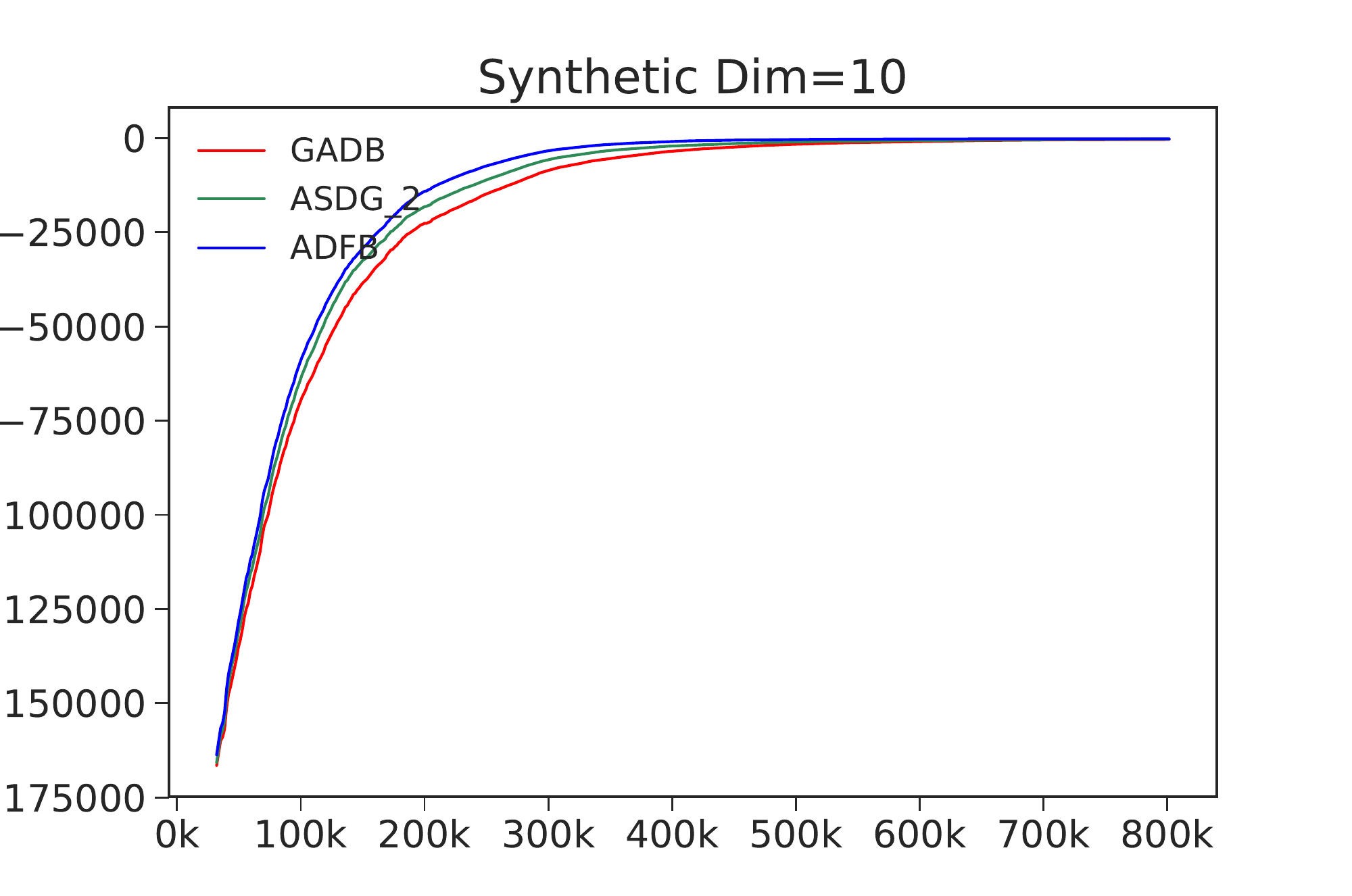}} \\
	\subfloat[Dim=20, K=4]{	\includegraphics[width=0.44\textwidth]{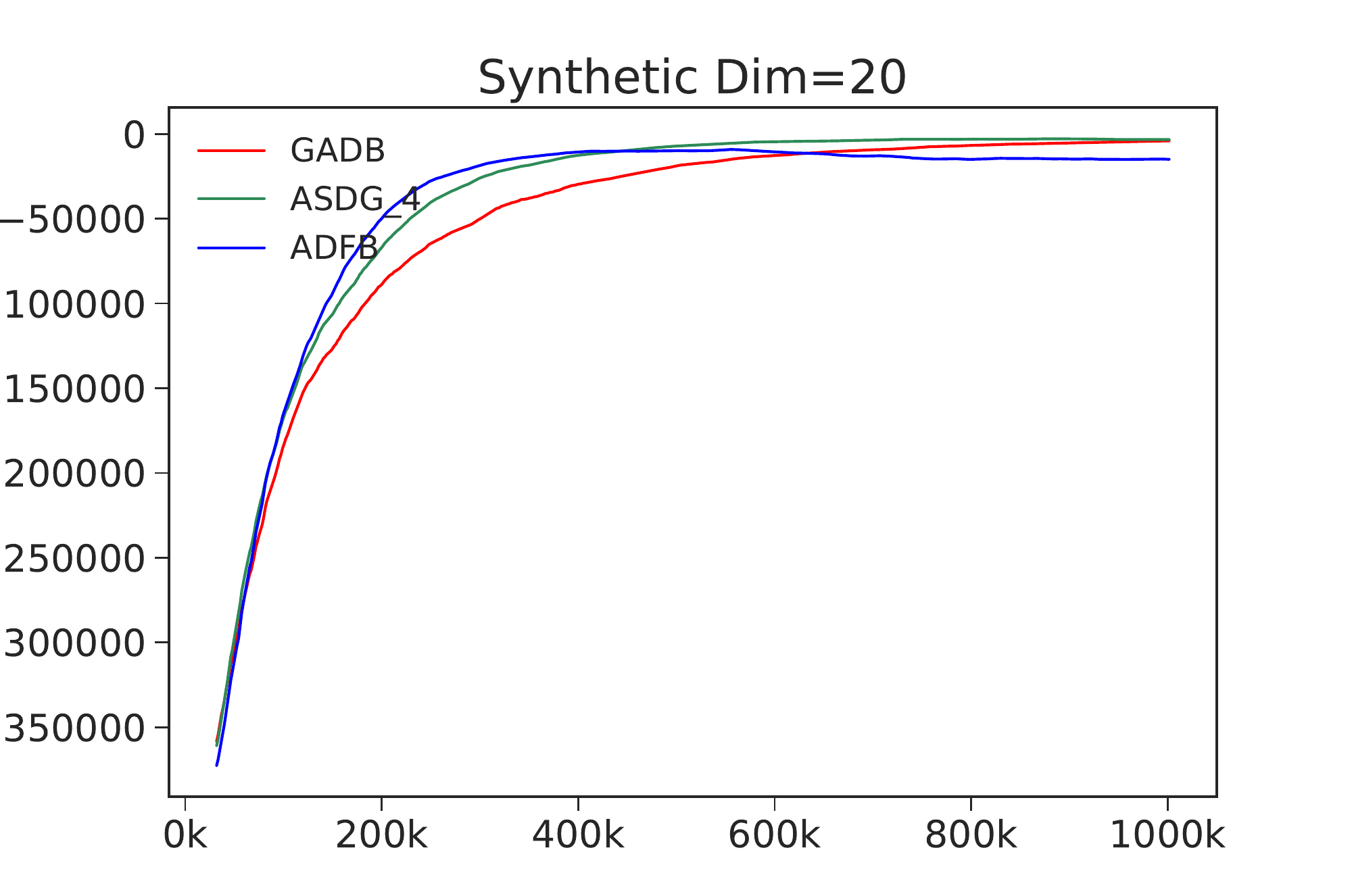}}
	\subfloat[Dim=40, K=4]{
	\includegraphics[width=0.44\textwidth]{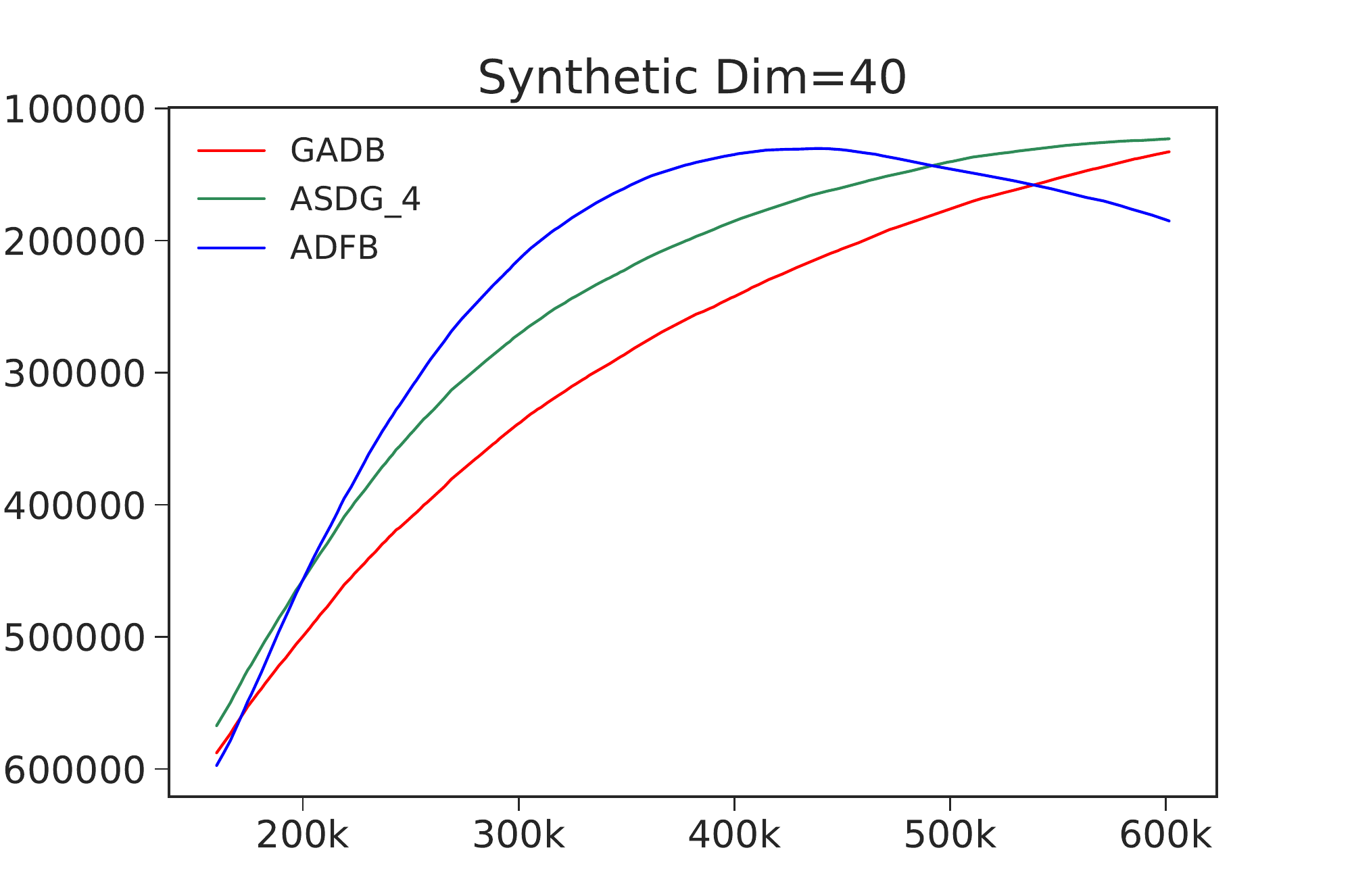}}
	\caption{Learning curve for synthetic high-dimensional 	continuous control tasks, varying from 4 to 40 dimensions. At high dimensions, our ASDG estimator provides an ideal balance between the accuracy (i.e., GADB) and efficiency (i.e., ADFB).}
	\label{synthetic}
\end{figure*}

Fig.~\ref{synthetic} shows the results on the synthetic environment for ASDG with different dimensions $m$ and number of subspaces $K$. The legend \textit{ASDG\_K} stands for our ASDG estimator with $K$ blocks assumption. For environments with relatively low dimensions such as (a) and (b), all of the algorithms converge to the same point because of the simplicity of the settings. Both ASDG and ADFB (that incorporates RB) outperform GADB significantly in terms of sample efficiency while ADFB is marginally better ASDG. For high dimensional settings such as (c) and (d), both ASDG and GADB converge to the same point with high accuracy. Meanwhile, ASDG achieves the convergence significantly faster because of its efficiency. ADFB, though having better efficiency, fails to achieve the competitive accuracy.

We observe an ideal balance between accuracy and efficiency. On the one hand, ASDG trades marginal accuracy for efficiency when efficiency is the bottleneck of the training, as is in (a) and (b). On the other hand, ASDG trades marginal efficiency for accuracy when accuracy is relatively hard to achieve, as is in (c) and (d). ASDG's tradeoff results in the combination of both the merits of its extreme cases.

We also demonstrate that the performance is robust to the assumed $K$ value in (a) when accuracy is not the major difficulty. As is shown in 
(a), the performance of ASDG is only decided by its sample efficiency, which is monotonically increased with $K$. However in complicated environments, an improper selection of $K$ may result in the loss of accuracy. Hence, in general, ASDG performs best overall when the $K$ value is set to the right value instead of the maximum. 

% Continuous Control
\subsection{OpenAI Gym's MuJoCo Environments}
We present the results of the proposed POSA algorithm with ASDG estimator on common benchmark tasks. These tasks and experiment settings have been widely studied in the deep reinforcement learning community \cite{duan2016benchmarking,gu2016q,wu2018variance,liu2018action-dependent}. We test POSA on several environments with high action dimensions, namely Walker2d, Hopper, HalfCheetah, and Ant, shown in Fig.~\ref{mujoco} and Fig.~\ref{walker}. In general, ASDG outperforms ADFB and GADB consistently but performs extraordinarily well for HalfCheetah.
Empirically, we find the block diagonal assumption \eqref{as2} for the advantage function is minimally violated, and that may be one of the reasons behind its good performance.
\begin{figure*}[htb]
\centering 
\begin{minipage}{.32\textwidth}
	\includegraphics[width=\textwidth]{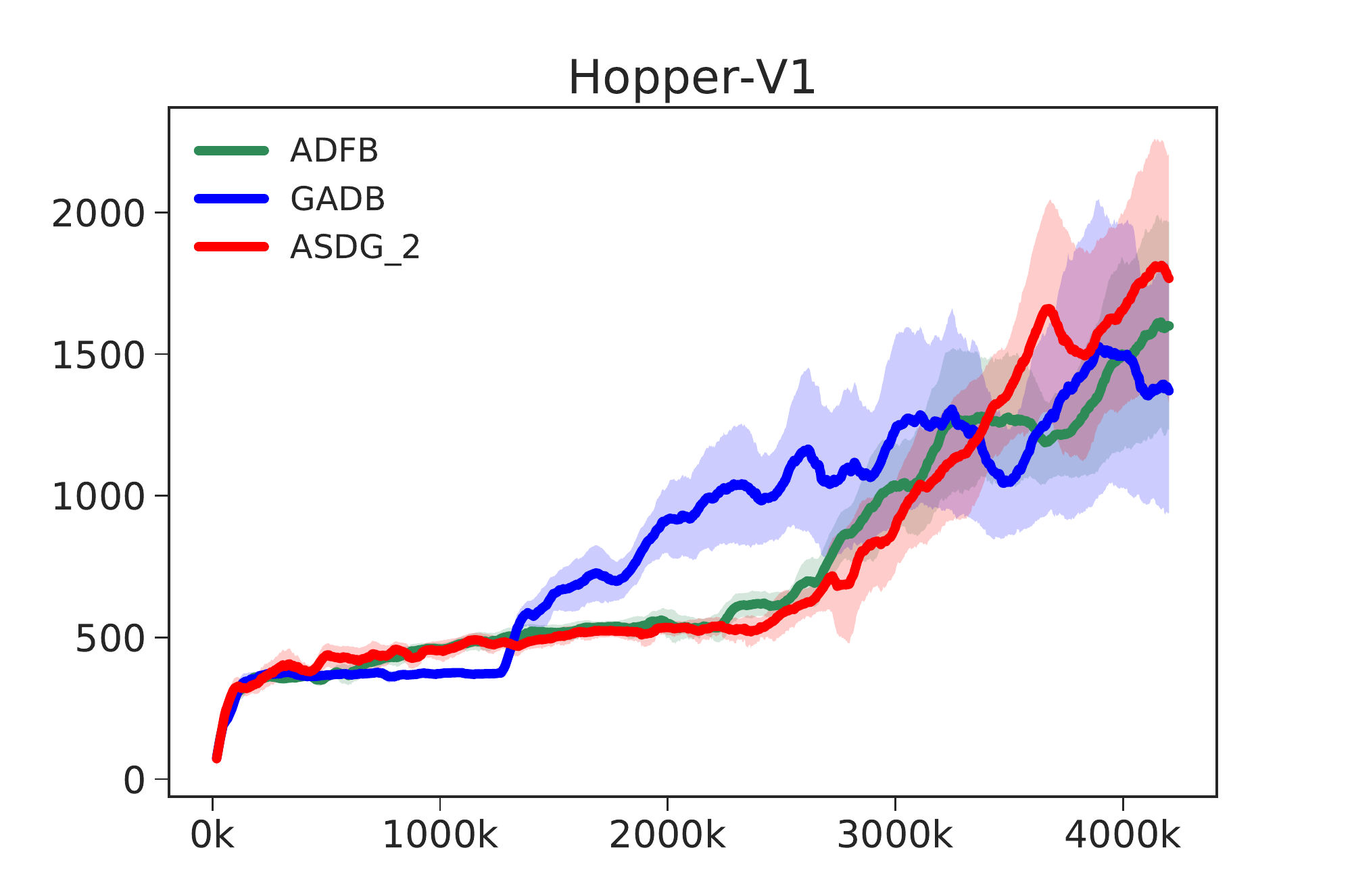}
\end{minipage}
\begin{minipage}{.32\textwidth}
	\includegraphics[width=\textwidth]{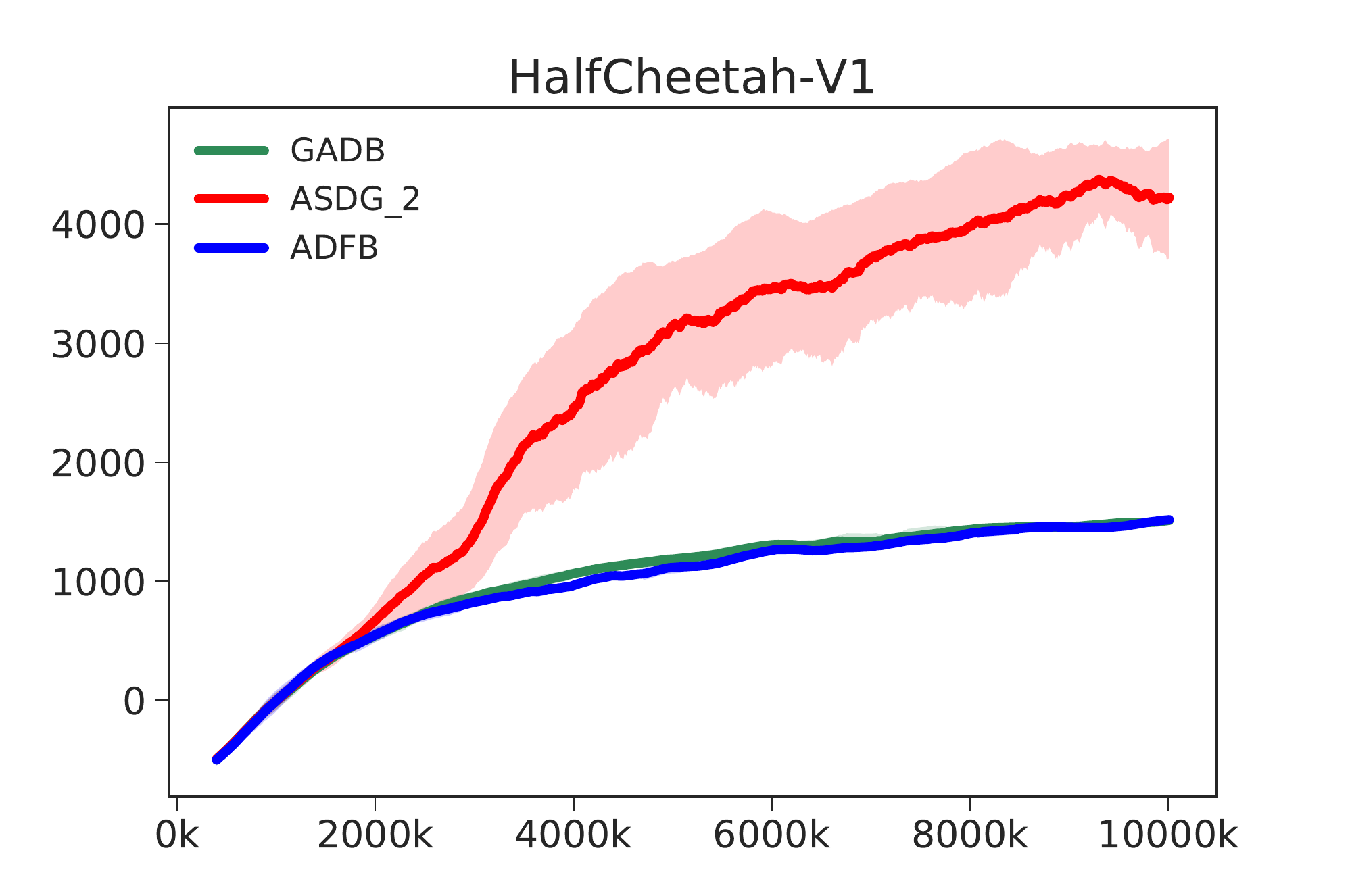}
    %{img/dim4_low.pdf}
\end{minipage}
\begin{minipage}{.32\textwidth}
	\includegraphics[width=\textwidth]{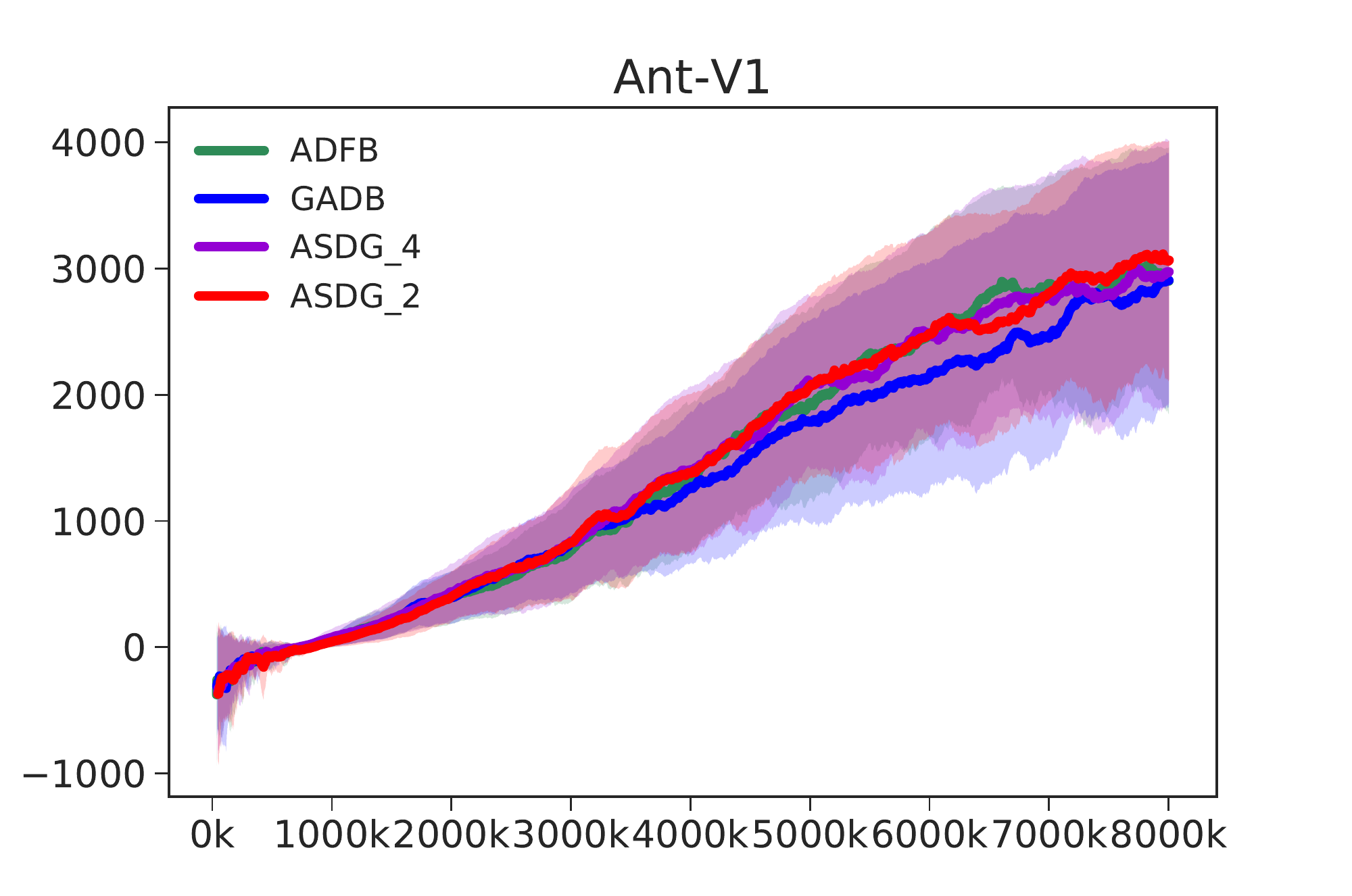}
\end{minipage}
\caption{Comparison between two baselines (ADFB, GADB) and our ASDG estimator on various OpenAI Gym Mujoco continuous control tasks, including Hopper-V1 (Dim=3), HalfCheetah-V1 (Dim=6) and Ant-V1 (Dim=8). Our ASDG estimator performs consistently the best across all these tasks.}
\label{mujoco}
\end{figure*}

\begin{figure}[htb]
\centering 
\includegraphics[width=3.5in]{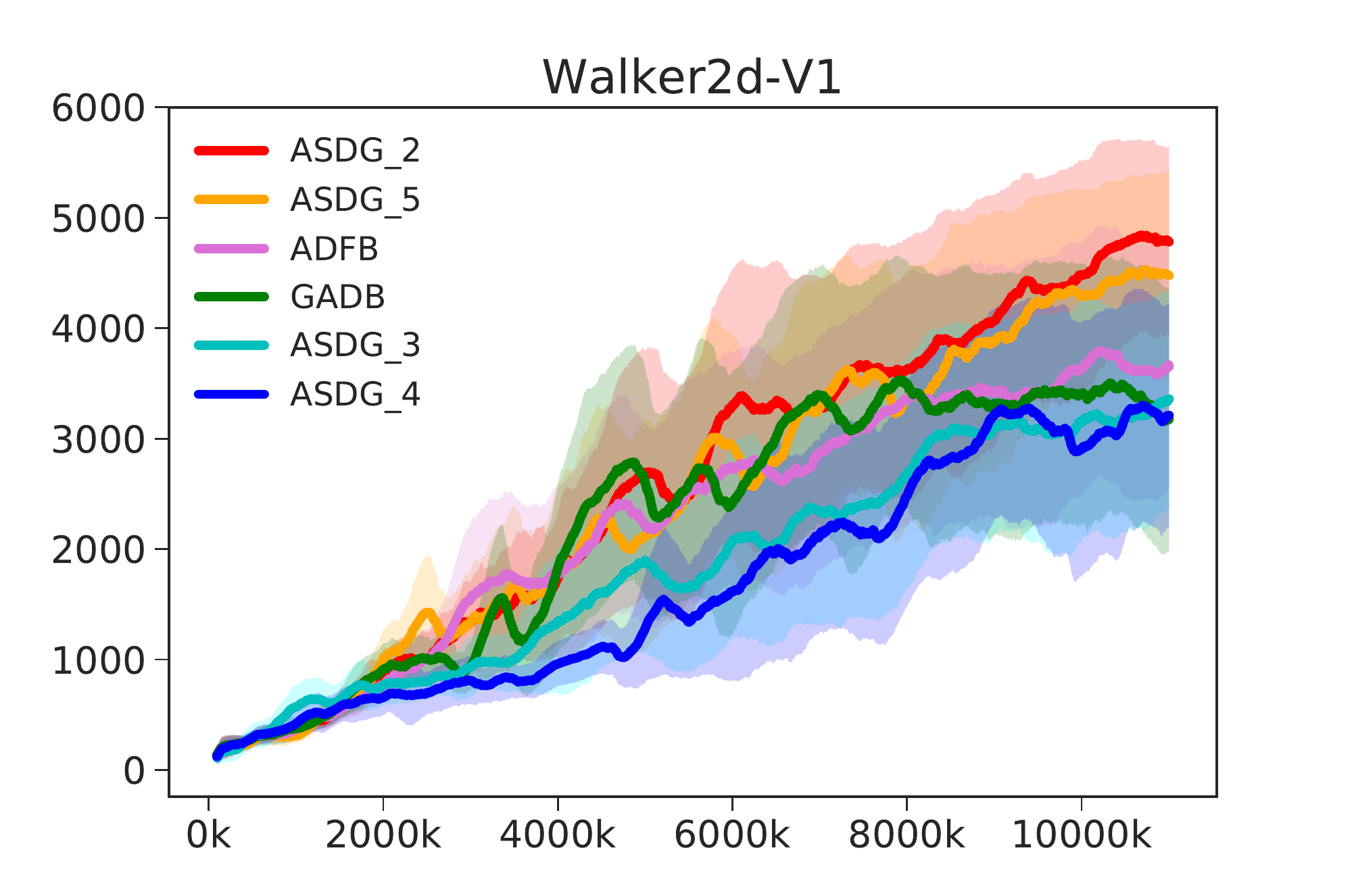}
\caption{The choices of action subspace number $K$ in the Walker2d-V1 environment.}
\label{walker}
\end{figure}

To investigate the choice of $K$, we test all the possible $K$ values in Walker2d. The optimal $K$ value is supposed to be between its extreme $K=1$ and $K=m$ cases. Empirically, we find it effective to conduct a grid search. We consider the automatically approach to finding the optimal $K$ value an interesting future work.

\section{Conclusion}

We propose action subspace dependent gradient (ASDG) estimator, which combines Rao-Blackwell theorem and Control Variates theory into a unified framework to cope with the high dimensional action space. 
We present policy optimization with second-order advantage information (POSA), which captures the second-order information of the advantage function via the wide \& deep architecture and exploits the information to find the dependency structure for ASDG. 
ASDG reduces the variance from the original policy gradient estimator while keeping it unbiasedness under relatively weaker assumptions than previous studies \cite{wu2018variance}. 
POSA with ASDG estimator performs well on a variety of environments including high-dimensional synthetic environment and OpenAI Gym's MuJoCo continuous control tasks. 
It ideally balances the two extreme cases and demonstrates the merit of both the methods.
\bibliographystyle{plain}
\bibliography{arxiv18}

\end{document}